\documentclass[11pt]{article}
\usepackage[a4paper,margin=27mm]{geometry}
\usepackage{amsmath,amssymb,amsthm,mathtools,bm}
\usepackage{microtype}
\usepackage[hidelinks]{hyperref}
\usepackage{enumitem}
\usepackage{booktabs}

\newtheorem{theorem}{Theorem}[section]
\newtheorem{proposition}[theorem]{Proposition}

\newtheorem{remark}[theorem]{Remark}
\newtheorem{definition}[theorem]{Definition}
\newcommand{\D}{\mathbb D}
\newcommand{\R}{\mathbb R}
\newcommand{\skewop}[1]{[#1]_{\times}}
\newcommand{\eD}{\varepsilon_{0}}
\newcommand{\eT}{\varepsilon}
\newcommand{\uvec}[1]{\underline{#1}}

\title{Dual Park--Ravani Interpolation of Rigid Motions:\\
Acceleration-Field Continuity and Holonomic Hermite Repair}
\author{Daniel Condurache\\
\small ``Gheorghe Asachi'' Technical University of Ia\c{s}i, Romania}
\date{}

\hypersetup{
  pdftitle={Dual Park--Ravani Interpolation of Rigid Motions: Acceleration-Field Continuity and Holonomic Hermite Repair},
  pdfauthor={Daniel Condurache},
  pdfsubject={Invariant interpolation of rigid motions with orthogonal dual tensors},
  pdfkeywords={Park--Ravani spline, rigid-motion interpolation, orthogonal dual tensors, hyper-dual algebra, acceleration field, holonomy, Hermite interpolation}
}

\begin{document}
\maketitle

\begin{abstract}
The Park--Ravani construction generates a twice continuously differentiable,
frame-invariant spline on $SO(3)$ by exponentiating cubic canonical-coordinate
polynomials. We show that the complete construction transfers, without a
change of form, to the group of orthogonal dual tensors, a representation of
rigid displacements. The transferred recurrence is stated compactly through
the dual extension of the right Jacobian of the exponential map and its first
Fr\'echet derivative. This yields interpolation of prescribed rigid poses and
continuity of the body dual twist and its first derivative. Using the
higher-order rigid-body kinematics of dual spatial twists, we then prove that
the resulting curve has a continuous physical acceleration field, not merely
a continuous quantity obtained by formally differentiating the dual part of a
twist. We also distinguish algebraic dual transfer from temporal differential
prolongation: their simultaneous first-order use takes place in a hyper-dual
algebra, and interpolation of arbitrary prolonged nodal data need not be
holonomic. A noncommuting three-pose example verifies the recurrence, all knot
continuity statements, and dimensional covariance under a change from metres
to millimetres. We define and analyse the first-order holonomy defect of a generic
hyper-dual interpolant, exhibit an exact counterexample, and remove the defect
by cubic or quintic Hermite interpolation in dual logarithmic coordinates.
\end{abstract}


\section{Introduction}

Park and Ravani introduced an efficient invariant interpolation scheme on
$SO(3)$ in which each segment is the exponential of a cubic polynomial in
canonical coordinates \cite{ParkRavani1997}. The construction interpolates an
ordered set of rotations, is invariant under changes of fixed and moving
frames, and produces a $C^2$ rotation curve. Its relation to minimum angular
acceleration is approximate in general and exact only in the special cases
identified in the original paper; no stronger variational assertion is used
here.

Invariant rotation and rigid-motion interpolation has also been developed
through quaternion curves \cite{Shoemake1985,Kim1995}, Lie-group versions of
the de Casteljau algorithm \cite{Crouch1999}, and interpolation or projection
schemes on $SE(3)$ \cite{Zefran1998,Belta2002}. Dual-quaternion blending and
screw interpolation provide a further computational representation of rigid
displacements \cite{Kavan2008}. These approaches establish the broader
geometric setting, but they do not supply the dual-tensor continuation of the
Park--Ravani recurrence or the acceleration-field and holonomy results proved
below. Unlike projection or blending methods, the present construction
retains the Park--Ravani recurrence in canonical coordinates and transfers it
exactly to a dimensionally homogeneous dual-tensor representation; no
projection back to $SE(3)$ is required.

Rigid displacement, however, possesses a natural dual description. An
orthogonal dual tensor combines rotation and translation without placing
dimensionless and length-valued entries in the same real matrix norm. This
suggests a direct question: does the Park--Ravani construction itself extend
from rotations to rigid motions by dual continuation? The answer is yes. More
precisely, every algebraic and analytic relation used by the construction
extends to the dual algebra, provided the logarithm is evaluated in one
admissible local branch \cite{CondurachePopa2025}.
The representation used here builds on dual-tensor formulations of rigid
motion \cite{ConduracheBurlacu2014} and on the classical principle of
transference \cite{Kotelnikov1895,Study1903,Rooney1975}. Hyper-dual arithmetic
is used only where the directional derivative of the Jacobian must be
evaluated exactly, in the automatic-differentiation sense of
\cite{FikeAlonso2012,CohenShoham2016}.

A second question is subtler. The dual part of the time derivative of a dual
twist is often informally called a linear acceleration. In general this is
incorrect. The derivative of the spatial dual twist determines a helicoidal
field of reduced accelerations; the complete physical acceleration field is
determined by the spatial dual twist together with its first derivative
\cite{Condurache2022}. Consequently, the physically meaningful smoothness
statement for a dual Park--Ravani spline must be proved at the level of the
entire point-acceleration field.

The contribution is not a formal substitution of dual quantities into a real
algorithm. The paper establishes: (i) an exact transfer theorem for the complete
Park--Ravani construction on orthogonal dual tensors; (ii) a compact
nonuniform-knot recurrence using the right Jacobian and its directional
derivative; (iii) continuity of the complete physical acceleration field of
every material point, rather than only of a formally differentiated dual
component; (iv) a precise distinction between dual extension and temporal
differential transformation, including an exact nonholonomy counterexample;
and (v) minimal-degree cubic and quintic dual-logarithmic Hermite repairs for
prescribed endpoint twist and acceleration-field data. Thus transfer,
physical interpretation, and holonomic repair form a single closed argument.

Theorem~\ref{thm:transfer} gives a precise algorithmic form of the classical
principle of transference for the complete Park--Ravani construction. Its
scope is deliberately algebraic: analytic identities transfer from real to
dual data. Proposition~\ref{prop:nonholonomy} identifies the boundary of this
statement. After an independent temporal nilpotent unit is introduced,
coefficientwise algebraic transfer does not ensure that the resulting
hyper-dual curve is the temporal jet of its base projection. Algebraic
transference and temporal holonomy are therefore distinct properties.

\section{Dual orthogonal tensors and rigid displacements}

\subsection{Dual vectors and tensors}

Let
\[
  \D=\R[\eD]/(\eD^2)
\]
be the algebra of dual numbers. Throughout the paper, an underline denotes a
dual quantity:
\[
  \uvec{x}=x+\eD x^{\circ},\qquad
  \uvec{A}=A+\eD A^{\circ}.
\]
The unit $\eD$ is the rigid, or geometric, nilpotent unit. It is dimensionless
as an algebraic generator; physical dimensions reside in the coefficients.
For example, the real part of a dual rotation vector is angular, whereas its
dual part has the dimension of length.

The transpose, scalar product, cross product, and tensor product are extended
by dual bilinearity. In particular,
\[
 (x+\eD x^{\circ})\times(y+\eD y^{\circ})
 =x\times y+\eD(x\times y^{\circ}+x^{\circ}\times y).
\]
Hence $\uvec{x}\times\uvec{x}=0$.

\begin{definition}
An orthogonal dual tensor is a dual tensor $\uvec{R}$ satisfying
\[
 \uvec{R}^{T}\uvec{R}=I,\qquad \det\uvec{R}=1.
\]
The resulting group is denoted by $SO(3,\D)$.
\end{definition}

Every rigid displacement $(R,p)$ admits the representation
\begin{equation}\label{eq:dual-pose}
 \uvec{R}=(I+\eD\skewop{p})R
          =R+\eD\skewop{p}R,
 \qquad R\in SO(3),\quad p\in\R^3.
\end{equation}
Composition and inversion are ordinary dual-tensor multiplication and
transpose. Thus $SO(3,\D)$ is a dimensionally homogeneous realization of the
rigid-displacement group.

The Lie algebra consists of dual skew tensors. We identify
$\skewop{\uvec{x}}$ with the dual vector $\uvec{x}$. Exponential and logarithm
are defined by analytic continuation of their real tensor series. An
\emph{admissible logarithm branch} is a local inverse of the exponential
selected around a dual logarithm $\uvec{x}$ whose real rotation-vector norm is
not a nonzero multiple of $2\pi$ \cite{CondurachePopa2025}. The value
$\|\operatorname{Re}\uvec{x}\|=\pi$ is admissible after fixing one of the two
axis signs: this is a branch choice, not a singularity of the exponential
Jacobian. Pure translation is included by the regular zero-angle
continuation. On such a branch,
\[
 \exp\skewop{\uvec{x}}\in SO(3,\D),\qquad
 \log\uvec{R}=\skewop{\uvec{x}}.
\]

\subsection{Body and spatial dual twists}

For a differentiable rigid motion $\uvec{R}(t)$, define
\begin{equation}\label{eq:twists}
 \skewop{\uvec{\omega}_{b}}
   =\uvec{R}^{T}\dot{\uvec{R}},
 \qquad
 \skewop{\uvec{\omega}_{s}}
   =\dot{\uvec{R}}\uvec{R}^{T}.
\end{equation}
The corresponding dual-vector relation is
\begin{equation}\label{eq:adjoint}
 \uvec{\omega}_{s}=\uvec{R}\uvec{\omega}_{b}.
\end{equation}
Differentiating and using
$\dot{\uvec{R}}\uvec{x}=\uvec{R}
(\uvec{\omega}_{b}\times\uvec{x})$ gives
\begin{equation}\label{eq:twist-derivative-transform}
 \dot{\uvec{\omega}}_{s}
 =\uvec{R}\dot{\uvec{\omega}}_{b}
  +\uvec{R}(\uvec{\omega}_{b}\times\uvec{\omega}_{b})
 =\uvec{R}\dot{\uvec{\omega}}_{b}.
\end{equation}
This cancellation will be essential at an interpolation knot.

\section{Temporal differential transform}

Let $\uvec{x}(t)$ be a dual-valued differentiable function. Its first temporal
differential transform is
\begin{equation}\label{eq:temporal-transform}
 \uvec{\breve{x}}(t)
 =\uvec{x}(t)+\eT\dot{\uvec{x}}(t),
 \qquad \eT^2=0.
\end{equation}
Here $\eT$ is distinct from the rigid dual unit $\eD$. The two units commute,
so the global coefficient algebra is
\begin{equation}\label{eq:hyperdual-algebra}
 \mathbb H_{\!D}
 =\R[\eD,\eT]/(\eD^2,\eT^2),
 \qquad \eD\eT\ne0.
\end{equation}
Indeed,
\[
 \uvec{\breve{x}}
 =x+\eD x^{\circ}+\eT\dot x
   +\eD\eT\dot x^{\circ}.
\]
Thus the first differential transform of a dual quantity is hyper-dual. No
hat notation is used at this level; it is reserved for higher-order
extensions.

Because \eqref{eq:temporal-transform} is exact first-order automatic
differentiation, analytic identities prolong functorially:
\begin{equation}\label{eq:functorial}
 \breve{\uvec{F(x)}}=F(\breve{\uvec{x}})
 =F(\uvec{x})+\eT\,DF(\uvec{x})[\dot{\uvec{x}}].
\end{equation}
This applies in particular to the exponential, the logarithm inside an
admissible chart, and the Jacobian used below. Equation
\eqref{eq:functorial} must not be confused with holonomy of an interpolated
hyper-dual curve: arbitrary transformed nodal data need not be the temporal
transform of the curve obtained by projecting their base parts.

\section{The dual Park--Ravani construction}

Let $t_0<\cdots<t_N$ be knot times, $h_i=t_{i+1}-t_i$, and let
$\uvec{R}_i\in SO(3,\D)$ be prescribed rigid poses. Choose one admissible
logarithm branch for every relative displacement, fixing an axis sign when
the real relative angle is $\pi$, and set
\begin{equation}\label{eq:relative-log}
 \skewop{\uvec{s}_i}
 =\log(\uvec{R}_i^T\uvec{R}_{i+1}).
\end{equation}
On the normalized interval $u=(t-t_i)/h_i\in[0,1]$, define
\begin{align}
 \uvec{r}_i(u)&=\uvec{a}_i u^3+\uvec{b}_i u^2+\uvec{c}_i u,
 \label{eq:cubic-coordinate}\\
 \uvec{R}(t)&=\uvec{R}_i
       \exp\skewop{\uvec{r}_i(u)}.\label{eq:segment}
\end{align}
Endpoint interpolation is equivalent to
\begin{equation}\label{eq:a-condition}
 \uvec{a}_i=\uvec{s}_i-\uvec{b}_i-\uvec{c}_i.
\end{equation}

For a dual vector $\uvec{r}$ define the dual right Jacobian
\begin{equation}\label{eq:right-J}
 \uvec{J}(\uvec{r})
 =I-\frac{1-\cos\uvec{q}}{\uvec{q}^{2}}\skewop{\uvec{r}}
 +\frac{\uvec{q}-\sin\uvec{q}}{\uvec{q}^{3}}
       \skewop{\uvec{r}}^{2},
 \qquad \uvec{q}^{2}=\uvec{r}\cdot\uvec{r},
\end{equation}
with the regular analytic continuation at $\uvec{q}=0$. Under the convention
\eqref{eq:twists}, the body dual twist of \eqref{eq:segment} is
\begin{equation}\label{eq:body-twist-segment}
 \uvec{\omega}_{b}(t)
 =\frac1{h_i}\uvec{J}(\uvec{r}_i)\uvec{r}'_i.
\end{equation}
Its time derivative is
\begin{equation}\label{eq:body-derivative-segment}
 \dot{\uvec{\omega}}_{b}(t)
 =\frac1{h_i^2}\left\{
 D\uvec{J}(\uvec{r}_i)[\uvec{r}'_i]\uvec{r}'_i
 +\uvec{J}(\uvec{r}_i)\uvec{r}''_i\right\}.
\end{equation}

Put
\begin{equation}\label{eq:t-u}
 \uvec{\tau}_i=3\uvec{a}_i+2\uvec{b}_i+\uvec{c}_i,
 \qquad
 \uvec{\nu}_i=6\uvec{a}_i+2\uvec{b}_i.
\end{equation}
These are the first two $u$-derivatives of $\uvec r_i$ at $u=1$.

\begin{proposition}[Nonuniform dual Park--Ravani recurrence]
\label{prop:recurrence}
Assume that the initial body dual twist
$\uvec{\omega}_{b,0}$ and its derivative
$\dot{\uvec{\omega}}_{b,0}$ are prescribed. Set
\begin{align}
 \uvec{c}_0&=h_0\uvec{\omega}_{b,0},&
 \uvec{b}_0&=\frac{h_0^2}{2}\dot{\uvec{\omega}}_{b,0},&
 \uvec{a}_0&=\uvec{s}_0-\uvec{b}_0-\uvec{c}_0.
 \label{eq:first-coefficients}
\end{align}
For $i=0,\ldots,N-2$, define successively
\begin{align}
 \uvec{c}_{i+1}
 &=\frac{h_{i+1}}{h_i}\,
   \uvec{J}(\uvec{s}_i)\uvec{\tau}_i,
 \label{eq:c-recurrence}\\
 2\uvec{b}_{i+1}
 &=\frac{h_{i+1}^2}{h_i^2}
 \left\{D\uvec{J}(\uvec{s}_i)[\uvec{\tau}_i]\uvec{\tau}_i
       +\uvec{J}(\uvec{s}_i)\uvec{\nu}_i\right\},
 \label{eq:b-recurrence}\\
 \uvec{a}_{i+1}
 &=\uvec{s}_{i+1}-\uvec{b}_{i+1}-\uvec{c}_{i+1}.
 \label{eq:a-recurrence}
\end{align}
Then \eqref{eq:segment} interpolates every $\uvec R_i$, and its body dual
twist and body dual-twist derivative are continuous at every internal knot.
\end{proposition}

\begin{proof}
At $u=0$, $\uvec J(0)=I$ and
$D\uvec J(0)[\uvec c]\uvec c=0$. Equations
\eqref{eq:body-twist-segment}--\eqref{eq:body-derivative-segment} therefore
give the initial data \eqref{eq:first-coefficients}. At $u=1$,
$\uvec r_i=\uvec s_i$, $\uvec r_i'=\uvec\tau_i$, and
$\uvec r_i''=\uvec\nu_i$. The left limits of the twist and its derivative
are consequently
\[
 \frac1{h_i}\uvec J(\uvec s_i)\uvec\tau_i,
 \qquad
 \frac1{h_i^2}\left\{
 D\uvec J(\uvec s_i)[\uvec\tau_i]\uvec\tau_i
 +\uvec J(\uvec s_i)\uvec\nu_i\right\}.
\]
The corresponding right limits on segment $i+1$ are
$\uvec c_{i+1}/h_{i+1}$ and
$2\uvec b_{i+1}/h_{i+1}^2$. Their equality is precisely
\eqref{eq:c-recurrence}--\eqref{eq:b-recurrence}. Finally,
\eqref{eq:a-recurrence} gives $\uvec r_{i+1}(1)=\uvec s_{i+1}$, hence pose
interpolation.
\end{proof}

\section{Exact dual transfer}

Write a complete set of real Park--Ravani input data as $Y$ and denote the
resulting real spline construction by $\mathcal P(Y)$. Its dual continuation
is not a new independently postulated algorithm.

\begin{theorem}[Dual transfer principle]\label{thm:transfer}
Let $Y+\eD Y^{\circ}$ be dual input data lying in a product of admissible
logarithm charts. Then
\begin{equation}\label{eq:transfer}
 \mathcal P_{\D}(Y+\eD Y^{\circ})
 =\mathcal P(Y)+\eD\,D\mathcal P(Y)[Y^{\circ}].
\end{equation}
In particular, every analytic identity used in the real Park--Ravani
construction, including its recurrence and frame covariance, holds over the
dual algebra with exactly the same form.
\end{theorem}

\begin{proof}
Addition, multiplication, contraction, and cross product extend by dual
bilinearity. For every analytic scalar or tensor function $F$,
\[
 F(X+\eD X^{\circ})=F(X)+\eD\,DF(X)[X^{\circ}],
\]
because $\eD^2=0$. The Park--Ravani construction is a finite composition of
these operations with exponential, local logarithm, and the analytic
functions entering $J$. The chain rule therefore gives
\eqref{eq:transfer}. Applying the same continuation to any real analytic
identity proves the final assertion.
\end{proof}

\begin{remark}
The theorem is an exact first-order transfer in the rigid dual direction. It
does not assert a variational minimum in a lexicographically ordered dual
functional. Such a statement requires a separate definition and proof.
\end{remark}

\begin{remark}[Forward propagation]
As in the original Park--Ravani construction, the recurrence propagates the
initial twist data forward and imposes no terminal condition. Over long knot
sequences, coefficient growth may therefore occur. The present paper
establishes the exact transferred recurrence and its physical continuity
properties, not a global conditioning result.
\end{remark}

\section{Continuity of the physical acceleration field}

Let the spatial dual twist and its derivative be decomposed as
\begin{equation}\label{eq:spatial-components}
 \uvec{\omega}_{s}=\omega+\eD v,
 \qquad
 \dot{\uvec{\omega}}_{s}=\dot\omega+\eD\dot v.
\end{equation}
The vector $\dot v$ is the reduced acceleration of the instantaneous point of
the rigid body passing through the origin. It is not, by itself, the physical
linear acceleration of a fixed material point. For a spatial point with
position vector $\rho$, define
\begin{align}
 \Phi_2&=\skewop{\dot\omega}+\skewop{\omega}^{2},
 \label{eq:Phi2}\\
 a_2&=\dot v+\omega\times v,
 \label{eq:a2}\\
 a_{\rho}&=a_2+\Phi_2\rho.
 \label{eq:physical-field}
\end{align}
Equation \eqref{eq:physical-field} is the physical acceleration field of the
rigid body \cite{Condurache2022}.

\begin{theorem}[Acceleration-field continuity]
\label{thm:physical-continuity}
The dual Park--Ravani spline generated by Proposition
\ref{prop:recurrence} has a continuous physical acceleration field at every
internal knot. More explicitly, for each fixed knot time $t_i$, the left and
right limits of $a_{\rho}$ coincide for every $\rho\in\R^3$.
\end{theorem}

\begin{proof}
Proposition \ref{prop:recurrence} gives equality at the knot of the left and
right limits of $\uvec R$, $\uvec\omega_b$, and
$\dot{\uvec\omega}_b$. Equations \eqref{eq:adjoint} and
\eqref{eq:twist-derivative-transform} imply equality of the corresponding
spatial pair
\[
 (\uvec\omega_s,\dot{\uvec\omega}_s).
\]
Taking real and dual parts in \eqref{eq:spatial-components} yields identical
left and right values of $\omega$, $v$, $\dot\omega$, and $\dot v$.
Therefore \eqref{eq:Phi2} and \eqref{eq:a2} give identical $\Phi_2$ and $a_2$.
Substitution in \eqref{eq:physical-field} proves equality of $a_\rho$ for
every spatial point $\rho$.
\end{proof}

\begin{remark}
Continuity of the dual-twist derivative alone is not the appropriate physical
interpretation. For $n=2$, the pair
$(\uvec\omega_s,\dot{\uvec\omega}_s)$ uniquely determines the complete
physical acceleration field; the dual Park--Ravani recurrence makes this pair
continuous.
\end{remark}

\section{Holonomy and its Hermite repair}\label{sec:holonomy}

\subsection{The first-order holonomy condition}

Let a curve with coefficients in the hyper-dual algebra
\eqref{eq:hyperdual-algebra} be written as
\begin{equation}\label{eq:generic-prolonged-curve}
 \uvec{\breve R}(t)=\uvec R_0(t)+\eT\uvec R_1(t).
\end{equation}
Its base projection is $\pi_0\uvec{\breve R}=\uvec R_0$. If
$\uvec{\breve R}$ is orthogonal over the hyper-dual algebra, then
$\uvec R_0\in SO(3,\D)$ and $\uvec R_0^T\uvec R_1$ is dual skew. This
tangency condition is necessary but does not say that the $\eT$ coefficient
is the time derivative of the base curve.

\begin{definition}[First-order holonomy]
The curve \eqref{eq:generic-prolonged-curve} is temporally holonomic if
\begin{equation}\label{eq:holonomy-condition}
 \boxed{\uvec R_1(t)=\dot{\uvec R}_0(t)}
\end{equation}
throughout its interval. Equivalently,
$\uvec{\breve R}=\uvec R_0+\eT\dot{\uvec R}_0$ is the temporal differential
transform of its base projection.
\end{definition}

Define the tensor defect and its body-trivialized dual vector by
\begin{align}
 \uvec\Delta(t)&=\uvec R_1(t)-\dot{\uvec R}_0(t),
 \label{eq:tensor-defect}\\
 \skewop{\uvec\delta_b(t)}&=\uvec R_0(t)^T\uvec\Delta(t).
 \label{eq:body-defect}
\end{align}
For an orthogonal hyper-dual curve, the right-hand side of
\eqref{eq:body-defect} is dual skew. Holonomy is equivalent to
$\uvec\delta_b\equiv0$. Under a constant change of fixed frame the body defect
is unchanged; under a constant change of moving frame it transforms by the
corresponding dual orthogonal action.

\begin{proposition}[Algebraic interpolation does not imply holonomy]
\label{prop:nonholonomy}
Applying the dual Park--Ravani formulas coefficientwise to individually
prolonged nodal data does not, in general, produce the temporal differential
transform of the spline obtained from the base data.
\end{proposition}

\begin{proof}
It is sufficient to work in a one-parameter, hence commutative, subgroup. Let
$u=t\in[0,1]$ and fix a nonzero dual vector $\uvec e$. Write
\[
 \uvec R_0(u)=\exp\skewop{x(u)\uvec e}.
\]
Take base endpoint coordinates $x(0)=0$, $x(1)=1$ and zero initial first and
second derivatives. The Park--Ravani cubic is
\begin{equation}\label{eq:counter-base}
 x(u)=u^3.
\end{equation}
Attach zero temporal coefficients to both endpoint coordinates and to the
initial derivative data. Coefficientwise application of the same construction
gives
\begin{equation}\label{eq:counter-prolonged}
 \breve x(u)=u^3+\eT\,0.
\end{equation}
The temporal differential transform of the base interpolant is instead
\begin{equation}\label{eq:counter-holonomic}
 x(u)+\eT x'(u)=u^3+3\eT u^2.
\end{equation}
Thus the scalar defect is $-3u^2$, and the tensor defect is nonzero for every
$u\in(0,1]$. The failure already occurs in a commutative subgroup, so
noncommutativity is not its cause.
\end{proof}

\begin{remark}
Each nodal object in the counterexample is individually a legitimate temporal
transform. The obstruction is global: the independently assigned nodal
coefficients are not the jet of the base spline selected by the interpolation
algorithm. Functorial prolongation \eqref{eq:functorial} and holonomy are
therefore distinct statements.
\end{remark}

\subsection{Cubic dual-logarithmic Hermite interpolation}

Consider one interval $[t_i,t_{i+1}]$, put
$u=(t-t_i)/h_i$, and let $\uvec s_i$ be given by
\eqref{eq:relative-log}, with
$\|\operatorname{Re}\uvec s_i\|\notin
2\pi\mathbb Z\setminus\{0\}$. Suppose the body dual twists
$\uvec\omega_{b,i}$ and $\uvec\omega_{b,i+1}$ are prescribed. Define
\begin{equation}\label{eq:cubic-log-data}
 \uvec d_0=h_i\uvec\omega_{b,i},\qquad
 \uvec d_1=h_i\uvec J(\uvec s_i)^{-1}\uvec\omega_{b,i+1}.
\end{equation}
The inverse in \eqref{eq:cubic-log-data} is well defined in the selected
admissible logarithm chart. Indeed, if
$\uvec J=J+\eD J^{\circ}$ and the real Jacobian $J$ is nonsingular, then
\begin{equation}\label{eq:dual-J-inverse}
 \uvec J^{-1}=J^{-1}-\eD J^{-1}J^{\circ}J^{-1}.
\end{equation}
The cubic dual-vector polynomial
\begin{align}
 \uvec r_H(u)={}&(3u^2-2u^3)\uvec s_i
 +(u^3-2u^2+u)\uvec d_0\notag\\
 &+(u^3-u^2)\uvec d_1
 \label{eq:cubic-hermite}
\end{align}
satisfies
\[
 \uvec r_H(0)=0,\quad \uvec r_H(1)=\uvec s_i,
 \quad \uvec r_H'(0)=\uvec d_0,
 \quad \uvec r_H'(1)=\uvec d_1.
\]

\begin{theorem}[Holonomic cubic repair]\label{thm:cubic-repair}
The curve
\begin{equation}\label{eq:cubic-base-curve}
 \uvec R_H(t)=\uvec R_i\exp\skewop{\uvec r_H(u)}
\end{equation}
interpolates the two dual poses and the two prescribed body dual twists. Its
temporal differential transform
\begin{equation}\label{eq:cubic-prolongation}
 \uvec{\breve R}_H(t)=\uvec R_H(t)+\eT\dot{\uvec R}_H(t)
\end{equation}
is identically holonomic.
\end{theorem}

\begin{proof}
The endpoint values of $\uvec r_H$ give the two poses. At $u=0$,
$\uvec J(0)=I$, and \eqref{eq:body-twist-segment} gives
$\uvec\omega_b(t_i)=\uvec d_0/h_i=\uvec\omega_{b,i}$. At $u=1$ it gives
\[
 \uvec\omega_b(t_{i+1})
 =h_i^{-1}\uvec J(\uvec s_i)\uvec d_1
 =\uvec\omega_{b,i+1}.
\]
Equation \eqref{eq:cubic-prolongation} satisfies
\eqref{eq:holonomy-condition} by construction, so its defect
\eqref{eq:tensor-defect} vanishes identically.
\end{proof}

For the scalar data in Proposition \ref{prop:nonholonomy}, with zero endpoint
velocities, the repaired base coordinate and its prolongation are
\[
 x_H(u)=3u^2-2u^3,
 \qquad
 \breve x_H(u)=3u^2-2u^3+\eT(6u-6u^2).
\]
They have zero defect and match both prescribed endpoint velocities.

\subsection{Quintic repair for prescribed accelerations}

If the body dual-twist derivatives are prescribed at both endpoints, set
\begin{align}
 \uvec e_0&=h_i^2\dot{\uvec\omega}_{b,i},
 \label{eq:e0-data}\\
 \uvec e_1&=\uvec J(\uvec s_i)^{-1}
 \left\{h_i^2\dot{\uvec\omega}_{b,i+1}
 -D\uvec J(\uvec s_i)[\uvec d_1]\uvec d_1\right\}.
 \label{eq:e1-data}
\end{align}
Introduce the quintic Hermite basis
\begin{align*}
 H_{01}&=10u^3-15u^4+6u^5,\\
 H_{10}&=u-6u^3+8u^4-3u^5,\\
 H_{11}&=-4u^3+7u^4-3u^5,\\
 H_{20}&=\tfrac12(u^2-3u^3+3u^4-u^5),\\
 H_{21}&=\tfrac12(u^3-2u^4+u^5),
\end{align*}
and define
\begin{equation}\label{eq:quintic-hermite}
 \uvec r_Q(u)=H_{01}\uvec s_i+H_{10}\uvec d_0
 +H_{11}\uvec d_1+H_{20}\uvec e_0+H_{21}\uvec e_1.
\end{equation}

\begin{theorem}[Holonomic quintic repair]\label{thm:quintic-repair}
The curve
\[
 \uvec R_Q(t)=\uvec R_i\exp\skewop{\uvec r_Q(u)}
\]
interpolates the endpoint dual poses, body dual twists, and body dual-twist
derivatives. Its temporal differential transform is identically holonomic.
Consequently, the prescribed physical acceleration fields are attained at both
endpoints.
\end{theorem}

\begin{proof}
The quintic basis gives
\[
 (\uvec r_Q,\uvec r_Q',\uvec r_Q'')(0)
 =(0,\uvec d_0,\uvec e_0),\qquad
 (\uvec r_Q,\uvec r_Q',\uvec r_Q'')(1)
 =(\uvec s_i,\uvec d_1,\uvec e_1).
\]
Equations \eqref{eq:body-twist-segment} and
\eqref{eq:body-derivative-segment}, together with
\eqref{eq:cubic-log-data}--\eqref{eq:e1-data}, give the prescribed endpoint
twists and derivatives. Temporal prolongation of this single base curve makes
\eqref{eq:holonomy-condition} an identity. The endpoint physical acceleration
fields follow from \eqref{eq:physical-field}.
\end{proof}

\begin{remark}
The cubic and quintic constructions solve different interpolation problems.
The cubic matches pose and twist. The quintic matches pose, twist, and twist
derivative, hence the complete physical acceleration field. The latter is not
claimed to be the unique possible holonomic repair; it is the minimal-degree
polynomial in one logarithm chart satisfying these six endpoint conditions.
\end{remark}

\section{Numerical verification}\label{sec:numerical}

\subsection{Data and implementation}

We consider three prescribed rigid poses at the nonuniform knot times
\[
 t_0=0,\qquad t_1=1,\qquad t_2=2.5,
 \qquad h_0=1,\quad h_1=1.5.
\]
The real rotation vectors $q_i$, in radians, and translations $p_i$, in
metres, are
\begin{equation}\label{eq:numerical-nodes}
\begin{array}{c|rrr|rrr}
 i&q_{i1}&q_{i2}&q_{i3}&p_{i1}&p_{i2}&p_{i3}\\ \hline
0& 0.08&-0.04& 0.06&0.05&-0.02& 0.03\\
1& 0.38&-0.22& 0.29&0.42&-0.16& 0.27\\
2&-0.18& 0.46& 0.34&0.83& 0.31&-0.12
\end{array}
\end{equation}
and the poses are formed as
\[
 R_i=\exp\skewop{q_i},\qquad
 \uvec R_i=(I+\eD\skewop{p_i})R_i.
\]
The initial body dual twist and its derivative are prescribed as
\begin{align}
 \uvec\omega_{b,0}
 &= (0.24,-0.13,0.19)
   +\eD(0.18,0.07,-0.11),\label{eq:num-initial-twist}\\
 \dot{\uvec\omega}_{b,0}
 &= (0.11,0.08,-0.06)
   +\eD(-0.09,0.14,0.05).\label{eq:num-initial-derivative}
\end{align}
Angular components use seconds as the time unit; the dual components carry
the corresponding metre-based dimensions.

The two relative logarithms obtained from \eqref{eq:relative-log} are
\begin{align}
 \uvec s_0={}&(0.2992674465,-0.1796622622,0.2312136665)\notag\\
 &+\eD(0.3750816019,-0.1514565323,0.2248304621),
 \label{eq:num-s0}\\
 \uvec s_1={}&(-0.4427966639,0.7577255960,-0.0168351043)\notag\\
 &+\eD(0.5945254558,0.2816054523,-0.3798654220).
 \label{eq:num-s1}
\end{align}
These rotations do not share a common axis, so the test is genuinely
noncommutative.

All analytic functions of dual tensors are evaluated by their exact dual
extensions. In particular, the quantity
\[
 D\uvec J(\uvec s_i)[\uvec\tau_i]\uvec\tau_i
\]
in \eqref{eq:b-recurrence} is computed by a second nilpotent unit $\eT$:
the coefficient of $\eT$ in
$\uvec J(\uvec s_i+\eT\uvec\tau_i)\uvec\tau_i$ is extracted. Thus the
calculation uses the hyper-dual algebra \eqref{eq:hyperdual-algebra} and no
finite-difference approximation.

The Python scripts reproducing every numerical value reported in this section
are provided as ancillary files. A single driver script executes the metre
and millimetre experiments, the dimensional-covariance check, and the
independent cubic and quintic Hermite verifications.

\subsection{Knot-continuity residuals}

The compact recurrence \eqref{eq:c-recurrence}--\eqref{eq:b-recurrence} was
also evaluated after expanding $D\uvec J$ into the trigonometric Park--Ravani
formula. Table \ref{tab:residuals} reports Euclidean norms obtained by stacking
the real and dual vector components, or Frobenius norms applied separately to
the real and dual tensor parts and then combining their numerical entries.
They are diagnostics in the fixed radian--metre--second convention, not
intrinsic norms on the dual Lie algebra. Dimensional covariance is assessed
separately in Section \ref{sec:dimensional-covariance}.

Most residuals in Tables~\ref{tab:residuals}--\ref{tab:accelerations} verify
mutually equivalent algebraic representations within the implementation. The
finite-difference differentiation of the complete dual-tensor curves in
Section~8.4 provides the independent curve-level check.

\begin{table}[ht]
\centering
\caption{Residuals at the internal knot $t_1$.}
\label{tab:residuals}
\begin{tabular}{@{}lr@{}}
\toprule
Verification & Residual\\
\midrule
Compact versus expanded recurrence & $8.33\times10^{-17}$\\
Dual pose interpolation & $8.05\times10^{-16}$\\
Body dual twist continuity & $1.36\times10^{-16}$\\
Body dual-twist derivative continuity & $2.29\times10^{-16}$\\
Spatial dual twist continuity & $2.00\times10^{-16}$\\
Spatial dual-twist derivative continuity & $2.29\times10^{-16}$\\
Maximum point-acceleration-field residual & $3.14\times10^{-16}$\\
\bottomrule
\end{tabular}
\end{table}

For an explicit field-level check, \eqref{eq:physical-field} was evaluated at
four points: $p_1$ and the three points obtained from $p_1$ by offsets
$0.25e_1$, $0.20e_2$, and $0.18e_3$ metres. The common acceleration values at
the knot are shown in Table \ref{tab:accelerations}. The maximum norm of the
left-minus-right difference is the last residual of Table
\ref{tab:residuals}.

\begin{table}[ht]
\centering
\caption{Physical point accelerations at $t_1$ in $\mathrm{m\,s^{-2}}$.}
\label{tab:accelerations}
\begin{tabular}{@{}crrr@{}}
\toprule
Point & $a_x$ & $a_y$ & $a_z$\\
\midrule
$p_1$          & 1.336338&-1.792470&1.773506\\
$p_1+0.25e_1$  & 1.284059&-1.754410&1.928334\\
$p_1+0.20e_2$  & 1.256146&-1.840993&1.788991\\
$p_1+0.18e_3$  & 1.267878&-1.844018&1.728332\\
\bottomrule
\end{tabular}
\end{table}

\subsection{Dimensional covariance}\label{sec:dimensional-covariance}

The entire experiment was repeated after replacing metres by millimetres.
Every translational datum, dual coefficient, and point offset was therefore
multiplied by $10^3$, whereas angles and times were unchanged. The real parts
of both relative logarithms remained identical. Their dual parts scaled by
$10^3$ with a maximum absolute discrepancy of
$2.27\times10^{-13}\,\mathrm{mm}$. The computed physical accelerations also
scaled by $10^3$; the maximum relative scaling error was
\begin{equation}\label{eq:unit-scaling-error}
 8.25\times10^{-16}.
\end{equation}
The largest knot residual in the millimetre computation was
$5.57\times10^{-13}$, consistent with the corresponding change in the
magnitude of the length-valued components.

This test is not a claim that a numerical norm is independent of units.
Rather, it verifies the correct covariance of each homogeneous component:
angular quantities remain unchanged and length-valued quantities acquire the
same scale factor as the unit of length. No norm combines a dimensionless
rotation matrix with a translation column.

\subsection{Independent curve-level check of the Hermite construction}

The cubic and quintic curves were constructed explicitly for a separate dual
rigid displacement with independently prescribed twists and twist derivatives
at both ends. Their dual tensor exponentials were then differentiated by
five-point, fourth-order finite-difference stencils. This verification is
independent of the $J$ and $DJ$ formulas used to generate the Hermite data.
The resulting pose, twist, and twist-derivative residuals are reported in
Table \ref{tab:hermite-curve-check}. The larger derivative residuals reflect
finite-difference truncation and cancellation rather than failure of the exact
endpoint identities.

\begin{table}[ht]
\centering
\caption{Independent curve-level verification of the Hermite interpolants.}
\label{tab:hermite-curve-check}
\begin{tabular}{@{}lrr@{}}
\toprule
Residual & Initial endpoint & Final endpoint\\
\midrule
Cubic pose & $0$ & $8.05\times10^{-16}$\\
Cubic body twist & $7.14\times10^{-13}$ & $4.04\times10^{-12}$\\
Quintic pose & $0$ & $8.05\times10^{-16}$\\
Quintic body twist & $1.59\times10^{-13}$ & $1.73\times10^{-12}$\\
Quintic body-twist derivative & $1.08\times10^{-10}$ & $6.32\times10^{-9}$\\
\bottomrule
\end{tabular}
\end{table}

The earlier algebraic round-trip residuals, retained in the software as
auxiliary implementation checks, remain between $2.40\times10^{-17}$ and
$3.40\times10^{-17}$. Direct evaluation of the five quintic basis polynomials
and their first two derivatives verifies all endpoint Kronecker conditions
exactly. For the scalar counterexample, the maximum defect is $3$, whereas the
repaired prolongation has identically zero defect.

\section{Conclusion and next steps}

The Park--Ravani construction transfers exactly from $SO(3)$ to orthogonal
dual tensors and therefore to rigid displacements. In compact form, its
nonuniform recurrence is governed by the dual right Jacobian and one
directional derivative. The transferred spline interpolates the prescribed
poses and preserves both the dual twist and its first derivative across knots.
The latter statement has been translated into its correct physical content:
the entire rigid-body acceleration field is continuous.

The noncommuting three-pose example verifies the compact recurrence, twist
continuity, acceleration-field continuity, and dimensional covariance under a
metre-to-millimetre change. The computations use exact hyper-dual directional
differentiation rather than finite differences. The first-order holonomy
problem has also been closed: coefficientwise interpolation of independently
prolonged nodal data is generically nonholonomic, while cubic and quintic
dual-logarithmic Hermite curves provide exact repairs for endpoint twist and
acceleration-field data, respectively. Higher temporal orders, their contact
defects, and holonomic Hermite interpolation of general rigid-motion jets are
left for future work.


\begin{thebibliography}{99}
\small

\bibitem{ParkRavani1997}
F. C. Park and B. Ravani,
``Smooth invariant interpolation of rotations,''
\emph{ACM Transactions on Graphics}, vol. 16, no. 3, pp. 277--295, 1997.

\bibitem{Shoemake1985}
K. Shoemake,
``Animating rotation with quaternion curves,''
\emph{Computer Graphics}, vol. 19, no. 3, pp. 245--254, 1985.

\bibitem{Kim1995}
M.-J. Kim, M.-S. Kim, and S. Y. Shin,
``A general construction scheme for unit quaternion curves with simple
high order derivatives,'' in \emph{Proceedings of SIGGRAPH}, pp. 369--376,
1995.

\bibitem{Crouch1999}
P. Crouch, G. Kun, and F. Silva Leite,
``The de Casteljau algorithm on Lie groups and spheres,''
\emph{Journal of Dynamics and Control Systems}, vol. 5, pp. 397--429, 1999.

\bibitem{Zefran1998}
M. \v{Z}efran, V. Kumar, and C. Croke,
``On the generation of smooth three-dimensional rigid body motions,''
\emph{IEEE Transactions on Robotics and Automation}, vol. 14, no. 4,
pp. 576--589, 1998.

\bibitem{Belta2002}
C. Belta and V. Kumar,
``An SVD-based projection method for interpolation on $SE(3)$,''
\emph{IEEE Transactions on Robotics and Automation}, vol. 18, no. 3,
pp. 334--345, 2002.

\bibitem{Kavan2008}
L. Kavan, S. Collins, J. \v{Z}\'ara, and C. O'Sullivan,
``Geometric skinning with approximate dual quaternion blending,''
\emph{ACM Transactions on Graphics}, vol. 27, no. 4, Art. 105, 2008.

\bibitem{ConduracheBurlacu2014}
D. Condurache and A. Burlacu,
``Dual tensors based solutions for rigid body motion parameterization,''
\emph{Mechanism and Machine Theory}, vol. 74, pp. 390--412, 2014.

\bibitem{Kotelnikov1895}
A. P. Kotelnikov,
\emph{Screw Calculus and Some Applications to Geometry and Mechanics},
Kazan, 1895 (in Russian).

\bibitem{Study1903}
E. Study, \emph{Geometrie der Dynamen}, Teubner, Leipzig, 1903.

\bibitem{Rooney1975}
J. Rooney,
``On the principle of transference,'' in \emph{Proceedings of the Fourth
World Congress on the Theory of Machines and Mechanisms}, Newcastle upon
Tyne, 1975.

\bibitem{FikeAlonso2012}
J. A. Fike and J. J. Alonso,
``Automatic differentiation through the use of hyper-dual numbers for second
derivatives,'' in \emph{Recent Advances in Algorithmic Differentiation},
Springer, pp. 163--173, 2012.

\bibitem{CohenShoham2016}
A. Cohen and M. Shoham,
``Application of hyper-dual numbers to multibody kinematics,''
\emph{Journal of Mechanisms and Robotics}, vol. 8, no. 1, Art. 011015, 2016.

\bibitem{CondurachePopa2025}
D. Condurache and I. Popa,
``Computation without singularity for logarithms of homogeneous matrices and
orthogonal dual tensors,''
\emph{Bulletin of the Polytechnic Institute of Ia\c{s}i}, vol. 71(75),
no. 1--2, pp. 39--56, 2025,
doi: 10.2478/bipmf-2025-0004.

\bibitem{Condurache2022}
D. Condurache,
``Higher-order relative kinematics of rigid body and multibody systems. A
novel approach with real and dual Lie algebras,''
\emph{Mechanism and Machine Theory}, vol. 176, Art. 104999, 2022,
doi: 10.1016/j.mechmachtheory.2022.104999.

\end{thebibliography}
\end{document}